\documentclass{article}

\usepackage{PRIMEarxiv}

\usepackage[utf8]{inputenc} 
\usepackage[T1]{fontenc}    
\usepackage{hyperref}       
\usepackage{url}            
\usepackage{booktabs}       
\usepackage{amsfonts}       
\usepackage{nicefrac}       
\usepackage{microtype}      
\usepackage{lipsum}
\usepackage{fancyhdr}       
\usepackage{graphicx}       
\graphicspath{{media/}}     

\usepackage{algorithm}
\usepackage{algpseudocode}
\usepackage{bbm}
\usepackage{natbib}
\usepackage{amsmath} 
\usepackage{amsthm}  
\usepackage{tikz}    
\usetikzlibrary{bayesnet, arrows.meta, positioning, shapes}

\usepackage{yansong}

\title{Distillation of Discrete Diffusion by Exact Conditional Distribution Matching
}

\author{
  Yansong Gao*, Yu Sun* \\
  \texttt{yansonggao@google.com,\ ysun258@wisc.edu,\ *equal contribution} 
  }

\begin{document}
\def \th {\theta}

\def \mm { \mathcal{M}}
\def \nn { \mathcal{N} }

\def \cc {\mathbb{C}}

\def \trans { \mathcal{T} }

\def \bset { \mathcal{B} }

\def \yvec{ \vec{y} }

\def \yvecs{ \vec{y^*} }
\def \dB {\text{d}_{\text{B}}}

\def \wb {{w_b}}
\def \wc {{w_l}}
{\maketitle}


\begin{abstract}
Discrete diffusion models (DDMs) are a powerful class of generative models
for categorical data, but they typically require thousands of function
evaluations for a single sample, making inference expensive. Existing
acceleration methods either rely on approximate simulators, such as
$\tau$-leaping, or on
distillation schemes that train new student models and auxiliary networks
with proxy objectives.  

We propose a simple and principled distillation alternative based on \emph{conditional
distribution matching}. Our key observation is that the reverse conditional
distribution of clean data given a noisy state, $p_{0\mid t}(x_0 \mid x_t)$,
admits a Markov decomposition through intermediate times and can be
recovered from marginal density ratios and the known forward CTMC kernel.
We exploit this structure to define distillation objectives that directly
match conditional distributions between a pre-trained teacher and a
low-NFE student, both for one-step and few-step samplers. 
\end{abstract}

\section{Introduction}
In recent years, discrete diffusion models (DDMs) have emerged as a milestone in modern generative modeling for categorical data \citep{lou2023discrete,austin2021structured,campbell2022continuous,ou2024your,meng2022concrete,gat2024discrete}. Unlike continuous diffusion models \citep{sohl2015deep,chen2022analog,ho2020denoising,nichol2021improved}, DDMs naturally accommodate data generation with discrete structures, e.g., language tokens \citep{he2023diffusionbert}, DNA sequences \citep{avdeyev2023dirichlet}, and images tokens \citep{hu2022global}. 

Despite this flexibility, DDMs often suffer from low sampling efficiency. They normally require a large number of function evaluations (NFEs), e.g., 1024 or more steps, making sampling computationally expensive. To reduce this cost, recent work has focused on
distilling DDMs into student models that admit faster sampling while preserving sampling quality \citep{zhu2025di,fu2025learnable}.

\paragraph{Contributions} We develop a conditional distribution matching viewpoint for distillation
in discrete diffusion models. This perspective leads to principled
distillation algorithms for both one-step and few-step generators. Our
methods match the conditional distribution of the clean data given noisy
states along the reverse trajectory, and they can be used to enhance an
existing sampler without training a separate auxiliary network.

\subsection{Related Work}

\paragraph{Efficient sampling in Discrete Diffusion models}
Various algorithms have been proposed to accelerate inference in DDMs while
maintaining sample quality. Approximate simulation methods, such as the
$\tau$-leaping algorithm \citep{campbell2022continuous}, are widely used
because they are simple to implement and amenable to parallelization.
$\tau$-leaping simulates the process by taking approximate Euler-like steps
that update all dimensions simultaneously and independently. Tweedie
$\tau$-leaping \citep{sun2022score,lou2023discrete} refines this idea by
specifying how the rate matrix changes with the noise schedule along the
reverse process, which improves accuracy at a given step size. More
recently, high-order numerical schemes tailored to discrete diffusion model
inference have been developed \citep{ren2025fast}. Although these
$\tau$-leaping variants provide substantial speedups and parallelism, their
approximation error still requires relatively small step sizes to achieve
high sampling quality.

\paragraph{Distillation for Discrete Diffusion Models.}

Distillation for DDMs is a rapidly developing area. Recent work on JYS
accelerates sampling by learning an optimized time discretization
\citep{park2024jump}. \citep{zhu2025di} distill a multi-step masked
diffusion model into a one-step generator by training a new student model
from scratch, using a proxy objective based on pseudo-intermediate states
and an auxiliary network that matches teacher–student conditional output
distributions. \citep{fu2025learnable} instead focus on few-step samplers
and introduce learnable sampler coefficients to improve efficiency. In
contrast, we use the conditional distribution matching perspective
described above to design distillation algorithms for both one-step and
few-step generators, and we directly refine an existing sampler rather than
training a separate student generator and auxiliary model.
\section{Preliminaries}
In this section, we review the basic concepts of discrete diffusion models
\subsection{Discrete Diffusion Processes}
We will be modeling a continuous time Markov chain (CTMC) over a finite support $\mathcal{X}$. The forward process describes how data distribution is corrupted. We denote the probability of transitioning from state $x \in \mathcal{X}$ at time $t$ to another state $y \in \mathcal{X}$ after an infinitesimal time interval $\Delta t$ by $p_{t+ \Delta t\mid t}(x_{t+ \Delta t} \mid x_t)$. \citep{campbell2022continuous} formalizes this  transition as:
\begin{equation}
  p_{t+\Delta t \mid t}(y \mid x) = \mathbf{1}_{ \{y = x \}} + Q_t(x,y)\,\Delta t + o(\Delta t),
  \label{eq:forward-ctmc}
\end{equation}
where $\mathbf{1}_{ \{ \cdot \}}$ is the indicator function,  $Q_t(x,y)$ is the $(x,y)$-th entry of the transition rate matrix
$Q_t$. ~\citep{campbell2022continuous} parameterize the transition rate matrix $Q_t$ as
$Q_t = \sigma(t)\,Q$, where $\sigma(t)$ is a
scalar function of time and $Q$ is a pre-defined base matrix with sparse structures. Let $p_t(x)$ denote the marginal distribution of state $x$ at time
$t$. In particular, $p_0(x)$ is the true data distribution of the state $x$. For a terminal time $T$, $p_T$ approaches a distribution $\pi$ depending on the base matrix $Q$. $\pi$ is normally either a uniform distribution, or a Dirac distribution that maps samples to a masked token state.

The inverse CTMC transports the data distribution from $p_T$ to $p_0$ as following
\begin{equation}
  p_{t-\Delta t \mid t}(y \mid x) = \mathbf{1}_{ \{y = x \}} +  \widetilde Q_t(x,y)\,\Delta t + o(\Delta t)
  \label{eq:reverse-ctmc}
\end{equation}
where $\widetilde Q_t$ is the reverse transition rate matrix, such that $\widetilde Q_t(x,x) = - \sum_{y \neq x} \widetilde Q_t(x,y) $ and $\widetilde Q_t(x,y) = \dfrac{p_t(y)}{p_t(x)}\,Q_t(y,x) $ in case $y \neq x$. 

The concrete score ratios $\frac{p_t(y)}{p_t(x)}$ are generally unknown. We therefore approximate them with a
neural score network $s_\theta : \mathcal{X} \times \mathbb{R} \to \mathbb{R}^{|\mathcal{X}|}$. In particular, \citet{lou2023discrete} introduce a score-entropy loss that trains $s_\theta$ such that
\begin{equation}
  s_\theta(x,t)
  \approx
  \Bigl(
    \frac{p_t(y)}{p_t(x)}
  \Bigr)_{y \neq x},
\end{equation}
which means, for each $(x,t)$, the network outputs the collection of
concrete scores for all $y \neq x$.

\subsection{Simulating Reverse Discrete Diffusion with Concrete Scores}

The multi-step sampling procedure of discrete diffusion models (DDMs)
approximates the reverse continuous-time Markov chain (CTMC).
Given a pre-trained score network $s^\theta(\cdot,\cdot)$, transition rate
matrices $\{Q_t\}_{t\in[0,T]}$, and an initial state
$X_T \sim \mathbf{p}_T$ at the initial time $T$, an Euler-type sampler iteratively
refines the sample by predicting the concrete score $s^\theta(t, x_t)$ and
then sampling a slightly less noisy state.

For two time steps $0<s<t<T$, let $p_{s\mid t}(x_s \mid x_t)$ denote the conditional probability distribution that moves a state from time $t$ to a state less noisy at time $s$.
Using a first-order Euler discretization of the CTMC, we can write
\begin{equation}
  p_{s\mid t}(x_s \mid x_t)
  \approx \mathbf{1}_{\{x_s = x_t\}}
  + (t-s)\, Q_t(x_t, x_s)\, s^\theta(t, x_t)_{x_s},
  \label{eq:euler-transition}
\end{equation}
where $\mathbf{1}_{\{\cdot\}}$ is the indicator function and
$s^\theta(t, x_t)_{x_s}$ denotes the component of the score corresponding
to state $x_s$. Repeating \eqref{eq:euler-transition} for decreasing
timesteps $T = t_K > t_{K-1} > \dots > t_1 > t_0 = 0$ yields the final
sample at time $t_0 = 0$.
\section{Method}
\label{sec: mmd}

\begin{algorithm}[t]
\caption{Conditional distribution matching distillation}
\label{alg:dmd}
\begin{algorithmic}[1]
\Require Pretrained concrete score model $s^\theta(\cdot,\cdot)$ (frozen during training); learnable student score model $s^\phi(\cdot,\cdot)$; target number of sampling steps $K$; loss weight function $w(t)$; dataset $\mathcal{D}$; transition rate matrices $\{Q_t\}_{t\in[0,T]}$.
\State Initialize student parameters $\phi \leftarrow \theta$.
\While{not converged}
  \State Sample a timestep $t \in [0, T]$ and a step size $\Delta t \sim \text{Uniform}\bigl(0, T/K\bigr)$.
  \State $s \gets \max(0,\, t - \Delta t) \quad$  \textcolor{OliveGreen}{ \# Set the intermediate time.}  
  \State $x_0 \sim \mathcal{D},\ x_t \gets \text{DiffusionForward}(x_0,\ t) \quad$  \textcolor{OliveGreen}{\# Sample a clean data and run the forward diffusion process.}
  \State $x_s \sim \mathbf{1}_{\{x_s = x_t\}} + (t-s)\, Q_t(x_t, x_s)\, s^\phi(t, x_t)_{x_s} \quad$ \textcolor{OliveGreen}{ \# Sample state at timestep $s$ using student scores.}
        
  \State $x_s \gets \operatorname{stopgrad}(x_s) \quad$ \textcolor{OliveGreen}{\# Detach the gradient through the sampled state.}
  \State $p^{\theta}_{0\mid s}(x_0 \mid x_s) \gets \sum_{y} \bigl[P_{t\mid 0}(x_s)^{-1}\bigr]_{x_0,y}\, s^\theta(t, x_s)_y \quad $ \textcolor{OliveGreen}{\# Estimate the teacher reverse conditional distribution.}
  \State $ p^{\phi}_{0\mid t}(x_0 \mid x_t) \gets \sum_{y} \bigl[P_{s\mid 0}(x_t)^{-1}\bigr]_{x_0,y}\, s^\phi(t, x_t)_y \quad$ \textcolor{OliveGreen}{ \# Estimate the student reverse conditional distribution. } 
  \State update $\phi$ by taking a stochastic gradient step on the cross entropy loss,
         \[
           \mathcal{L}(t, x_s, x_t)
           = - w(t)\!\sum_{x_0}
             p^{\theta}_{0\mid s}(x_0 \mid x_s)\,
             \log p^{\phi}_{0\mid t}(x_0 \mid x_t).
         \]
\EndWhile
\end{algorithmic}
\end{algorithm}

Our goal is to obtain a similar sampling quality with far fewer steps than
the original sampler requires. To this end, we fine-tune the pre-trained
score model $s^\theta(\cdot,\cdot)$ into a new \emph{student} score model
$s^\phi(\cdot,\cdot)$. The student sampler uses a strongly reduced number
of reverse steps. The student model $s^\phi$ may suffer from degraded sampling
quality when we reduce the number
of reverse steps and increase the step size. To avoid this, we proceed as follows. The student
model generates the final sample in a single step at time $t$, whereas the
teacher model generates the same sample in a single step at a smaller
intermediate time $s < t$. We then require that the student's one-step
generative process $p_{0\mid t}^\phi(x_0 \mid x_t)$ align with the teacher's one-step generative process $p_{0\mid s}^\theta(x_0 \mid x_s)$.

\medskip
\noindent\textbf{Decomposition of $p_{0\mid t}$.}
Let $p_{0\mid t}(x_0 \mid x_t)$ denote the (in general unknown) conditional
distribution of the clean state $x_0$ given a noisy state $x_t$ at time $t$.
Using basic probability identities, we can decompose this conditional as
\begin{align}
  p_{0\mid t}(x_0 \mid x_t)
  &= \sum_{x_s} p_{0,s\mid t}(x_0, x_s \mid x_t) \nonumber  \\
  &= \sum_{x_s} p_{s\mid t}(x_s \mid x_t)\,
                p_{0\mid s,t}(x_0 \mid x_s, x_t) \nonumber  \\
  &= \sum_{x_s} p_{s\mid t}(x_s \mid x_t)\,
      \frac{p_{0,s,t}(x_0, x_s, x_t)}{p_{s,t}(x_s, x_t)}.
      \label{eq:p0t-1}
\end{align}

Note the forward process is Markov in time. If conditioned on $x_s$,  the
future state $x_t$ is independent of $x_0$, so that
\begin{equation}
  p_{0,s,t}(x_0, x_s, x_t) = p_{0,s}(x_0,x_s)  p_{t\mid s}(x_t \mid x_s) = p_{0\mid s}(x_0 \mid x_s) p_s(x_s) p_{t\mid s}(x_t \mid x_s) = p_{0\mid s}(x_0 \mid x_s) p_{s,t}(x_s, x_t )\nonumber .
\end{equation}
Substituting this Markov property into \eqref{eq:p0t-1} yields the simpler
decomposition
\begin{equation}
  p_{0\mid t}(x_0 \mid x_t)
  = \sum_{x_s} p_{s\mid t}(x_s \mid x_t)\,
               p_{0\mid s}(x_0 \mid x_s).
  \label{eq:p0t-markov}
\end{equation}
Equation~\eqref{eq:p0t-markov} shows that
$p_{0\mid t}$ can be expressed as a mixture of intermediate 
conditionals $p_{0\mid s}(\cdot \mid x_s)$ with mixing weights
$p_{s\mid t}(x_s \mid x_t)$.
\medskip

The decomposition of $p_{0\mid t}(x_0 \mid x_t)$ as Equation~\eqref{eq:p0t-markov} follows directly from the Markov property. Howeverw, in practicee, all three factors
$p_{0\mid t}(x_0 \mid x_t)$, $p_{0\mid s}(x_0 \mid x_s)$, and
$p_{s\mid t}(x_s \mid x_t)$ are unknown and should be estimated using teacher and student models.

\medskip
\paragraph{Recovering $p_{0\mid t}$ from marginal ratios.}
We first derive an identity that expresses $p_{0\mid t}(\cdot\mid x)$ in
terms of the marginals $p_t$ and the known forward discrete diffusion kernel
$p_{t\mid 0}$.  Fix two states $x,y \in \mathcal{X}$ (possibly $x=y$).  Then
\begin{align}
  \frac{p_t(y)}{p_t(x)}
  = \sum_{x_0} \frac{p_{0,t}(x_0,y)}{p_t(x)} &= \sum_{x_0}
       \frac{p_{t\mid 0}(y\mid x_0)\, p_0(x_0)}{p_t(x)} \nonumber \\
  &= \sum_{x_0}
       \frac{p_{t\mid 0}(y\mid x_0)}{p_{t\mid 0}(x\mid x_0)}
       \frac{p_{t\mid 0}(x\mid x_0)\, p_0(x_0)}{p_t(x)} \nonumber \\
  &= \sum_{x_0}
       \frac{p_{t\mid 0}(y\mid x_0)}{p_{t\mid 0}(x\mid x_0)}\,
       p_{0\mid t}(x_0 \mid x),
  \label{eq:ratio-linear-system-clean}
\end{align}
where in the last line we used
$p_{0\mid t}(x_0\mid x)
 = p_{t\mid 0}(x\mid x_0)p_0(x_0)/p_t(x)$.

For each fixed $x$, define the $|\mathcal{X}|\times|\mathcal{X}|$ conditional ratios matrix
$P_{t\mid 0}(x)$ with entries
\[
  \bigl[P_{t\mid 0}(x)\bigr]_{y,x_0}
  := \frac{p_{t\mid 0}(y\mid x_0)}{p_{t\mid 0}(x\mid x_0)}.
\]
Let $\mathbf{r}_t(x) \in \mathbb{R}^{|\mathcal{X}|}$ denote the vector of ratios
$\mathbf{r}_t(x)_y = p_t(y)/p_t(x)$ and let
$\mathbf{p}_{0\mid t}(\cdot\mid x) \in \mathbb{R}^{|\mathcal{X}|}$ be the vector
whose $x_0$-th entry is $p_{0\mid t}(x_0\mid x)$.  Then
\eqref{eq:ratio-linear-system-clean} can be written compactly as
\[
  \mathbf{r}_t(x) = P_{t\mid 0}(x)\, \mathbf{p}_{0\mid t}(\cdot\mid x).
\]
Whenever the conditional ratios matrix $P_{t\mid 0}(x)$ is invertible, this linear system yields
\begin{equation}
  p_{0\mid t}(x_0\mid x)
  = \sum_{y}
      \bigl[P_{t\mid 0}(x)^{-1}\bigr]_{x_0,y}\,
      \frac{p_t(y)}{p_t(x)} .
  \label{eq:p0t-from-ratios-clean}
\end{equation}
Thus, if we can approximate the ratio vector $r_t(x)$, we can recover
$p_{0\mid t}(\cdot\mid x)$ up to the linear transform
$P_{t\mid 0}(x)^{-1}$, which depends only on the known discrete forward kernel.
\medskip

\subsection{Conditional Distribution Matching Distillation}

In our distillation setup, we use the student score network $s^\phi$ to estimate the marginal ratios at timestep $t$, so that $s^\phi(t,x_t)_y \approx p_t(y)/p_t(x_t)$.
Substituting this approximation into
\eqref{eq:p0t-from-ratios-clean} gives an estimate of the
backward conditional distribution,
\begin{equation}
  p^{\phi}_{0\mid t}(x_0\mid x_t)
  \approx \sum_{y \neq x_t}
    \bigl[P_{t\mid 0}(x_t)^{-1}\bigr]_{x_0,y}\,
    s^\phi(t,x_t)_y + \bigl[P_{t\mid 0}(x_t)^{-1}\bigr]_{x_0,x_t}.
  \label{eq:student-p0t-clean}
\end{equation}

The same construction is used for the teacher at the intermediate time
$s$.  Let $P_{s\mid 0}(x_s)$ be defined analogously using the forward
kernel from time $0$ to $s$, and let $s^\theta$ be the teacher score
network.  Then
\begin{equation}
  p^{\theta}_{0\mid s}(x_0\mid x_s)
  \approx \sum_{y \neq x_s}
    \bigl[P_{s\mid 0}(x_s)^{-1}\bigr]_{x_0,y}\,
    s^\theta(s,x_s)_y + \bigl[P_{s\mid 0}(x_s)^{-1}\bigr]_{x_0,x_s}.
  \label{eq:teacher-p0s-clean}
\end{equation}
For the transition probability $p_{s\mid t}(x_s\mid x_t)$ we use the
Euler-type approximation derived earlier for the reverse CTMC:
\begin{equation}
  p_{s\mid t}(x_s\mid x_t)
  \approx \mathbf{1}\{x_s = x_t\}
  + (t-s)\, Q_t(x_t,x_s)\, s^\phi(t,x_t)_{x_s},
  \label{eq:euler-ps-t-clean}
\end{equation}
where $Q_t$ is the rate matrix at time $t$ and
$s^\phi(t,x_t)_{x_s}$ is the $x_s$-th component of the student score.
\medskip


We approximate the right-hand side of the decomposition in
equation~\eqref{eq:p0t-markov} using a single Monte Carlo sample
$x_s \sim p_{s\mid t}(\cdot\mid x_t)$ together with
\eqref{eq:student-p0t-clean} and \eqref{eq:teacher-p0s-clean}.  
These estimations form the basis of our \emph{Conditional Distribution Matching
Distillation Algorithm~\ref{alg:dmd}}. A natural way to enforce this alignment between the two sides of equation~\eqref{eq:p0t-markov} is to minimize the Kullback–Leibler (KL) divergence, averaged over timesteps and states.

\subsection{Practical Implementation}
\label{subsec:Practical-Implementation}

To compute inverse of $P_{t\mid 0}(x)$, we first require a closed-form expression for
the forward CTMC transition probabilities $p_{t\mid 0}(x_t \mid x_0)$. These probabilities solve the Kolmogorov forward equation associated with the (possibly time-inhomogeneous) generator $Q_t$, and can be obtained by integrating this differential equation.

An analytical solution arises when the rate matrices $Q_t$ and $Q_{t'}$ commute for all $t,t'$. One convenient way to guarantee this is
to parameterize the generator as $Q_t = \sigma(t)\, Q$ where $\sigma(t)$ is a scalar function of time and $Q \in \mathbb{R}^{|\mathcal{X}|\times|\mathcal{X}|}$ is a fixed,
time-independent base matrix. Under this parameterization, the forward
transition kernel admits the explicit form
\begin{equation}
  p_{t\mid 0}(x_t = j \mid x_0 = i)
  = \bigl(S \exp\!\bigl[\Lambda \textstyle\int_0^t \sigma(s)\,ds\bigr]
     ]S^{-1}\bigr)_{i,j},
  \label{eq:qt0-eigendecomp}
\end{equation}
where $Q = S \Lambda S^{-1}$ is the eigendecomposition of $Q$, $\Lambda$ is
diagonal, and $\exp(\cdot)$ denotes the element-wise exponential applied to
the diagonal entries of $\Lambda$.

The base matrix $Q$ can be chosen in various ways. A simple and useful
choice is the ``uniform'' transition generator $Q = E - I$, where $E \in \mathbb{R}^{|\mathcal{X}|\times
|\mathcal{X}|}$ is the all-ones matrix and $I \in \mathbb{R}^{|\mathcal{X}|\times|\mathcal{X}|}$ is the identity
matrix.

\section{Conclusion}
We addressed the sampling inefficiency of discrete diffusion models by
introducing a conditional distribution matching framework for distillation.
Rather than designing new proxy objectives or auxiliary networks, we work
directly with the conditional distribution of clean data given noisy
states along the reverse diffusion trajectory. Starting from a Markov
decomposition of $p_{0\mid t}$ and a linear system linking this
conditional to marginal ratios and the forward CTMC kernel, we derived a
distillation objective that aligns teacher and student conditionals.  

Our resulting algorithm fine-tunes an existing score model into a
low-NFE student sampler, and applies both to one-step and few-step
generators. The method only requires access to the forward kernel and
score network, and can be implemented efficiently when the forward CTMC is
parameterized as $Q_t = \sigma(t) Q$ with a shared base matrix $Q$, for
which the transition probabilities admit a closed-form expression via
eigendecomposition. This leads to a simple, modular, and model-agnostic
procedure that can be plugged into a wide range of discrete diffusion
models.  

Future work includes extending our framework to more general time-varying
generators that do not commute, exploring adaptive choices of intermediate
times and loss weights, and combining conditional distribution matching
with advanced numerical schemes to further reduce the number of function
evaluations required for high-quality sampling.

\bibliographystyle{apalike}   
\bibliography{references}   

\newpage
\appendix
\section{Theoretical Analysis: Superiority of the Optimal Student}
\label{appendix:proofs}

In this section, we provide the theoretical justification for the distillation performance. We demonstrate that the optimal student model, derived via moment matching, acts as a conditional expectation estimator over the teacher's stochastic predictions. Consequently, by Jensen's inequality, the student achieves a reconstruction error (MSE) that is strictly bounded by the teacher's error.

\subsection{Optimal Estimators}

\paragraph{Teacher Model.} 
Consider the general diffusion training setup. The teacher model $\theta(x_t, t)$ is trained to denoise $x_t$ back to $x_0$ with loss

\begin{equation}
    \mathcal{L}_{\text{tch}}(\theta) = \int_0^1 \mathbb{E}_{x_0, x_t} [\|x_0 -  \theta(x_t, t)\|_2^2] dt.
\end{equation}

As established in \cite{ho2020denoising}, minimizing the standard variational lower bound (equivalently the MSE loss $\mathcal{L}_{\text{tch}}$) results in an optimal teacher $\theta^*$ that estimates the posterior mean of the data:
\begin{equation}
    \theta^*(x_t, t) = \mathbb{E}_{q(x_0|x_t)}[x_0]=\mathbb{E}[x_0 \mid x_t].
\end{equation}

\paragraph{Student Model.} 
Given $0\le t<s\le 1$. For the student model $\gamma(x_s, s)$, we minimize the distillation loss
\begin{equation}
    \mathcal{L}_{\text{std}}(\gamma, \theta) = \int_0^1 \int_0^s \mathbb{E}_{x_0, x_s, x_t} [\|\gamma(x_s, s) - \theta(x_t, t)\|_2^2] dt ds.
\end{equation}

The optimal student $\gamma^*$ minimizes this $L_2$ error with respect to the fixed teacher $\theta$. Therefore, the optimal student is the conditional expectation of the teacher's output given the student's input state $x_s$:
\begin{equation}
    \label{eq:opt_student}
    \gamma^*(\theta, x_s, s) = \mathbb{E}_{q(x_0|x_t, x_s)} [\theta(x_t, t)].
\end{equation}
This implies that the student learns to "average out" the stochastic variations in the teacher's predictions caused by different noise instantiations $x_t$ for a given time $s$.

\subsection{Proof of Theorem 1}

\begin{theorem}
\label{thm:student_superiority}
For any given pretrained diffusion model $\theta$ and its corresponding optimal student model $\gamma^*$ (defined in Eq. \ref{eq:opt_student}), the student model achieves a reconstruction error bounded by that of the teacher model for any timestep $s \in [0, 1]$:
\begin{equation}
    \mathbb{E}_{x_t|x_s} \left[ \|\gamma^*(x_s, s) - x_0\|_2^2 \right] \leq \mathbb{E}_{x_t|x_s} \left[ \|\theta(x_t, t) - x_0\|_2^2 \right].
\end{equation}
\end{theorem}

\begin{proof}
Let us analyze the reconstruction error of the optimal student (LHS). Substituting $\gamma^*(x_s, s) = \mathbb{E}_{x_t|x_s}[\theta(x_t, t)]$:
\begin{equation}
    \text{LHS} = \|\gamma^*(\theta, x_s, s) - x_0\|_2^2 = \left\| \mathbb{E}_{x_t|x_s}[\theta(x_t, t)] - x_0 \right\|_2^2.
\end{equation}
Since $x_0$ is constant with respect to the expectation over $x_t$ (conditioned on $x_s$ and fixed $x_0$), we can move it inside the expectation:
\begin{equation}
    \text{LHS} = \left\| \mathbb{E}_{x_t|x_s} [ \theta(x_t, t) - x_0 ] \right\|_2^2.
\end{equation}
We now invoke \textbf{Jensen's Inequality}. The squared Euclidean norm function $f(y) = \|y\|_2^2$ is convex. Jensen's inequality states that for a convex function $f$ and random variable $Y$, $f(\mathbb{E}[Y]) \leq \mathbb{E}[f(Y)]$. Letting $Y = \theta(x_t, t) - x_0$, we have:
\begin{equation}
    \left\| \mathbb{E}_{x_t|x_s} [ \theta(x_t, t) - x_0 ] \right\|_2^2 \leq \mathbb{E}_{x_t|x_s} \left[ \| \theta(x_t, t) - x_0 \|_2^2 \right].
\end{equation}
The term on the right is exactly the expected reconstruction error of the teacher model.
\begin{equation}
    \|\gamma^*(x_s, s) - x_0\|_2^2 \leq \mathbb{E}_{x_t|x_s} \left[ \mathcal{L}_{\text{MSE}}(\theta) \right].
\end{equation}
Thus, the optimal student is guaranteed to have an MSE less than or equal to the average MSE of the teacher. The inequality is strict whenever the teacher's prediction $\theta(x_t, t)$ has non-zero variance conditioned on $x_s$.
\end{proof}




\end{document}